\documentclass{article}

\usepackage{times}
\usepackage{graphicx} 
\usepackage{subfigure}

\usepackage{natbib}

\usepackage{algorithm}
\usepackage{algorithmic}

\usepackage{hyperref}

\usepackage{amssymb, amsmath, amsthm}
\newtheorem{lemma}{Lemma}
\newtheorem{theorem}{Theorem}

\DeclareMathOperator*{\argmin}{arg\,min}

\usepackage[accepted]{icml2014}

\icmltitlerunning{Learning Structures for Deep Neural Networks}

\begin{document}
\abovedisplayskip=3.0pt plus 0.0pt minus 3.0pt
\belowdisplayskip=3.0pt plus 0.0pt minus 3.0pt
\jot=-1.0pt


\twocolumn[
\icmltitle{Learning Structures for Deep Neural Networks}

\icmlauthor{Jinhui Yuan}{yuanjinhui@oneflow.org}
\icmladdress{OneFlow, Beijing, China}
\icmlauthor{Fei Pan, Chunting Zhou, Tao Qin, Tie-Yan Liu}{taoqin,tyliu@microsoft.com}
\icmladdress{Microsoft Research, Beijing, China}

\icmlkeywords{structure learning, deep neural networks, ICML}

\vskip 0.3in
]

\begin{abstract}
In this paper, we study the automatic structure learning for deep neural networks (DNN), motivated by the observations that the performance of a deep neural network is highly sensitive to its structure and previous successes of DNN heavily depend on human experts to design the network structures. We focus on the unsupervised setting for structure learning and propose to adopt the efficient coding principle, rooted in information theory and developed in computational neuroscience, to guide the procedure of structure learning without label information. This principle suggests that a good network structure should maximize the mutual information between inputs and outputs, or equivalently maximize the entropy of outputs under mild assumptions. We further establish connections between this principle and the theory of Bayesian optimal classification, and empirically verify that larger entropy of the outputs of a deep neural network indeed corresponds to a better classification accuracy. Then as an implementation of the principle, we show that sparse coding can effectively maximize the entropy of the output signals, and accordingly design an algorithm based on \emph{global group sparse coding} to automatically learn the inter-layer connection and determine the depth of a neural network. Our experiments on a public image classification dataset demonstrate that using the structure learned from scratch by our proposed algorithm, one can achieve a classification accuracy comparable to the best expert-designed structure (i.e., convolutional neural networks (CNN)). In addition, our proposed algorithm successfully discovers the local connectivity (corresponding to local receptive fields in CNN) and invariance structure (corresponding to pulling in CNN), as well as achieves a good tradeoff between marginal performance gain and network depth. All of this indicates the power of the efficient coding principle, and the effectiveness of automatic structure learning.
\end{abstract}

\section{Update in 2021}

This manuscript was written in January 2014 and once submitted to ICML 2014. Unfortunately, we did not continue this line of research and did not publish this article either. Today, we decide to publish it and expect the idea and empirical results can be helpful to those who would like to understand and investigate the problems such as why deep learning works. 

Since 2014, several related works emerge, among which the most close ones are the information bottleneck principle \cite{tishby2015deep} and maximum coding rate reduction principle \cite{chan2021redunet}. Both of these two principles are based on supervised learning which assumes the labels of data are known, while the principle of this paper investigate the structure learning through unsupervised learning. We assume the optimal structure of neural networks can be derived from the input features even without labels. Furthermore, let $x$ , $z$, $y$ indicate the input data, the learned representation and labels respectively. Information bottleneck principle \cite{tishby2015deep} aims to maximize the mutual information between $z$ and $y$, meanwhile minimize the mutual information between $x$ and $z$. However, the information maximization principle in this paper aims to maximize the mutual information between $x$ and $z$.  The maximum coding rate reduction principle \cite{chan2021redunet}, on the one hand, tries to maximize the mutual information between $x$ and $z$, which is essentially the same as information maximization principle. However, on the other hand, the $MCR^{2}$ principle also leverages the labels of data and minimize the volume of $z$ within each class. In linear model, information maximization principle leads to PCA (principal component analysis) while maximum coding rate reduction leads to LDA (linear discriminant analysis). 

\section{Introduction}
Recent years, people have witnessed a resurgence of neural networks in the machine learning community. Indeed, systems built on deep neural network techniques (DNN) demonstrate remarkable empirical performance in a wide range of applications. For examples, convolutional neural networks keeps the records in the ImageNet challenge ILSVRC 2012 \cite{krizhevsky_2012} and ILSVRC 2013 \cite{zeiler_2013}. The core of state-of-the-art speech recognition systems are also based on DNN techniques \cite{mohamed_2009,deng_2011}. In the applications to natural language processing, neural networks are making steady progress too \cite{collobert_2011,mikolov_2013}.

Empirical evidences for understanding why deep neural networks works so well are also accumulated. Besides the techniques such as dropout, local normalization, and data augmentation, the evidences suggest that the architecture or structure of neural networks plays a significant role in its success. For example, an early result reported by \cite{jarrett_2009} finds that, a two-layer architecture is always better than a single-layer one. More surprisingly, the paper observes that, given an appropriate structure, even random assignment of networks parameters can yield a decent performance. In addition, in the state-of-the-art systems such as \cite{krizhevsky_2012, zeiler_2013,mohamed_2009,deng_2011}, the network structure, in particularly, the inter-layer connection, the number of nodes in each layer, and the depth of the networks, are all designed by human experts in a very careful and probably painful way. This requires in-depth domain knowledge (e.g., the structure of convolutional neural networks \cite{lecun_1998} largely originates from the inspiration of biological nervous system \cite{hubel_1962,fukushima_1982}, and the network structure in \cite{deng_2011} heavily depend on the domain knowledge in speech recognition) or hundreds of times of trial-and-error \cite{jarrett_2009}.

Given this situation, a natural question arises: can we learn a good network structure for DNN from scratch in a fully automatic fashion? What is the principled way to achieve this? To the best of our knowledge, the studies on these important questions are still very limited. Only \cite{chen_2013}
shows the possibility of learning the number of nodes in each layer with a nonparametric Bayesian approach. However, there is still no attempt on the automatic learning of inter-layer connections and the depth of DNN. These are exactly the focus of our paper.

For this purpose, we borrow an important principle, called the efficient coding principle, from the domain of biological nervous systems, in which area there exists tremendous research work on understanding the structure of human brains. The principle basically says that a good structure (brain structure in their case and the structure of DNN in our case) forms an efficient internal representation of external environments \cite{barlow_1961,linsker_1988}. Rephrased by our familiar language, the principle suggests that the structure of a good network should match the statistical structure of input signals. In particular, it should maximize the mutual information between the inputs and outputs, or equivalently maximize the entropy of the output signals under mild assumptions. While the principle seems intuitive and a little informal, we show that it has a solid theoretical foundation in terms of Bayesian optimal classification, and thus has a strong connection with the optimality of the neural networks from the machine learning perspective. In particular, we first show that the principle suggests us to maximize the independence between the output signals. Then we notice that the top layer of any neural network is a \emph{softmax} linear classifier, and the independency between the nodes in the top hidden layer is a sufficient condition for the \emph{softmax} linear classifier to be the Bayesian optimal classifier. This theoretical foundation also provides us a clear way to determine the depth of the deep neural networks: if after multiple layers of non-linear transformations (learned under the guidelines of the efficient coding principle) the hidden nodes become statistically independent of each other, then there is no need to add another hidden layer (i.e., the depth of the network is finalized) since we have already been optimal in terms of the classification error.

We then investigate how to design structure learning algorithm based on the principle of efficient coding. We show that sparse coding can implement the principle under the assumption of zero-peaked and heavy-tailed prior distributions. Based on this discovery, we design an effective structure learning algorithm based on \emph{global group sparse coding}.  When customized for image analysis, we discuss how the proposed algorithm can learn inter-layer connections, handle invariance, and determine the depth.

We conduct a set of experiments on a widely used dataset for image classification. We have several interesting findings. First, although we have not imposed any prior knowledge onto the structure learning process, the DNN with our automatically learned structure can provide a very competitive classification accuracy, which is very close to the well-tuned CNN model that is designed by human experts. Second, our algorithm can automatically discover the local connection structure, simply due to the match of the statistical structure of input signals. Third, we notice that the pooling operation specifically designed in CNN can also be automatically implemented by our learning algorithm based on group sparse coding. All these results demonstrate the power of automatic structure learning based on the efficient coding principle.

While our work is just a preliminary step towards automatic structure learning, we have seen very positive signs suggesting that structure learning could be an important direction to better understand DNN, to further improve the performance of DNN, and to generalize the application scope of DNN based learning algorithms.

\section{The Principle for Structure Learning}
The key of unsupervised structure learning for DNN is to adopt an appropriate principle to guide the procedure
of structure learning. In this section, we describe our used principle and discuss its advantages for structure learning.

\begin{figure}
\centering
\subfigure[]{
\includegraphics[bb = 60 5 580 467,width=0.22\textwidth]{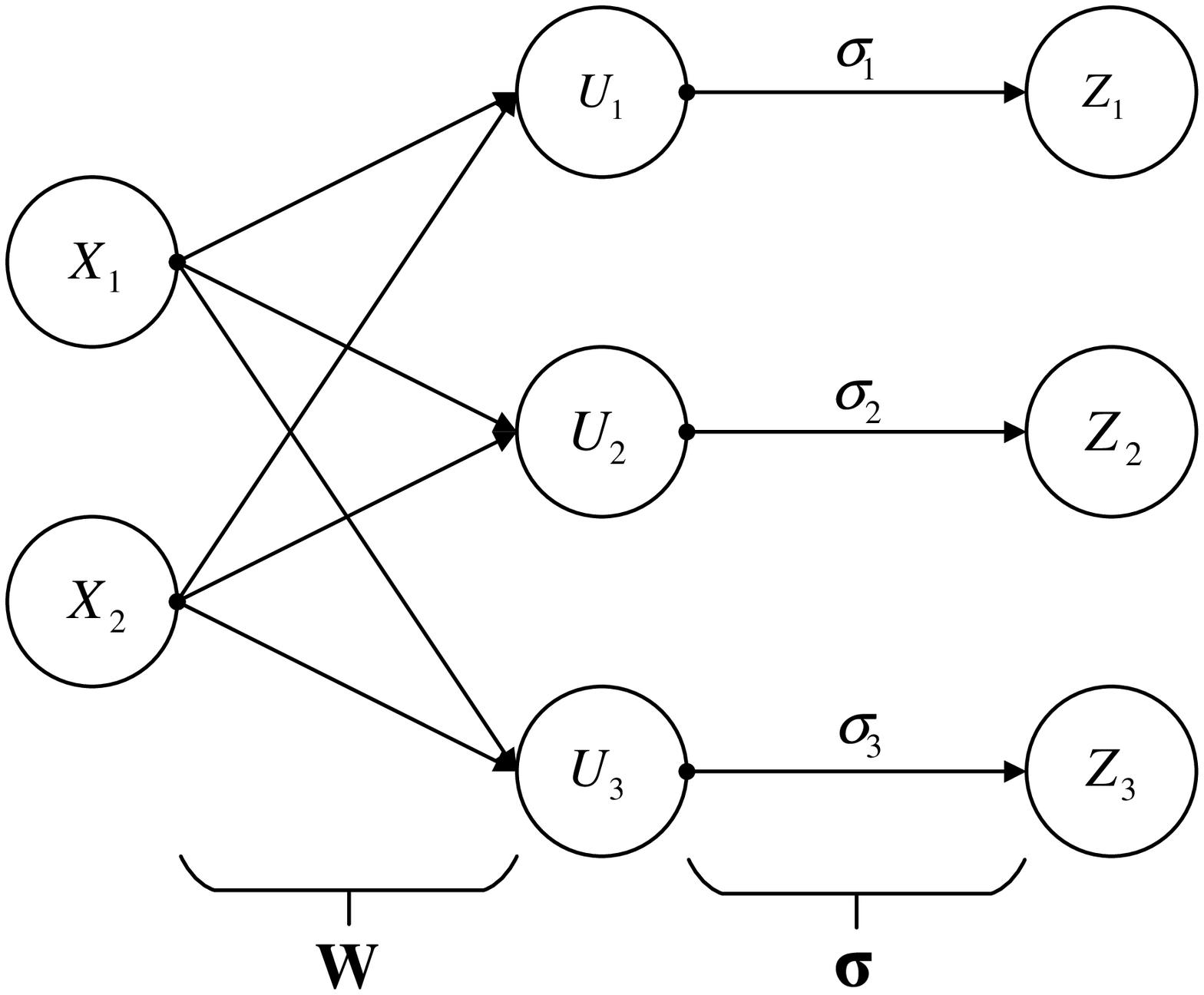}
\label{efficient_coding:sub1}
}
\hspace{0.03\textwidth}
\subfigure[]{
\includegraphics[bb = 95 325 439 739, width=0.18\textwidth]{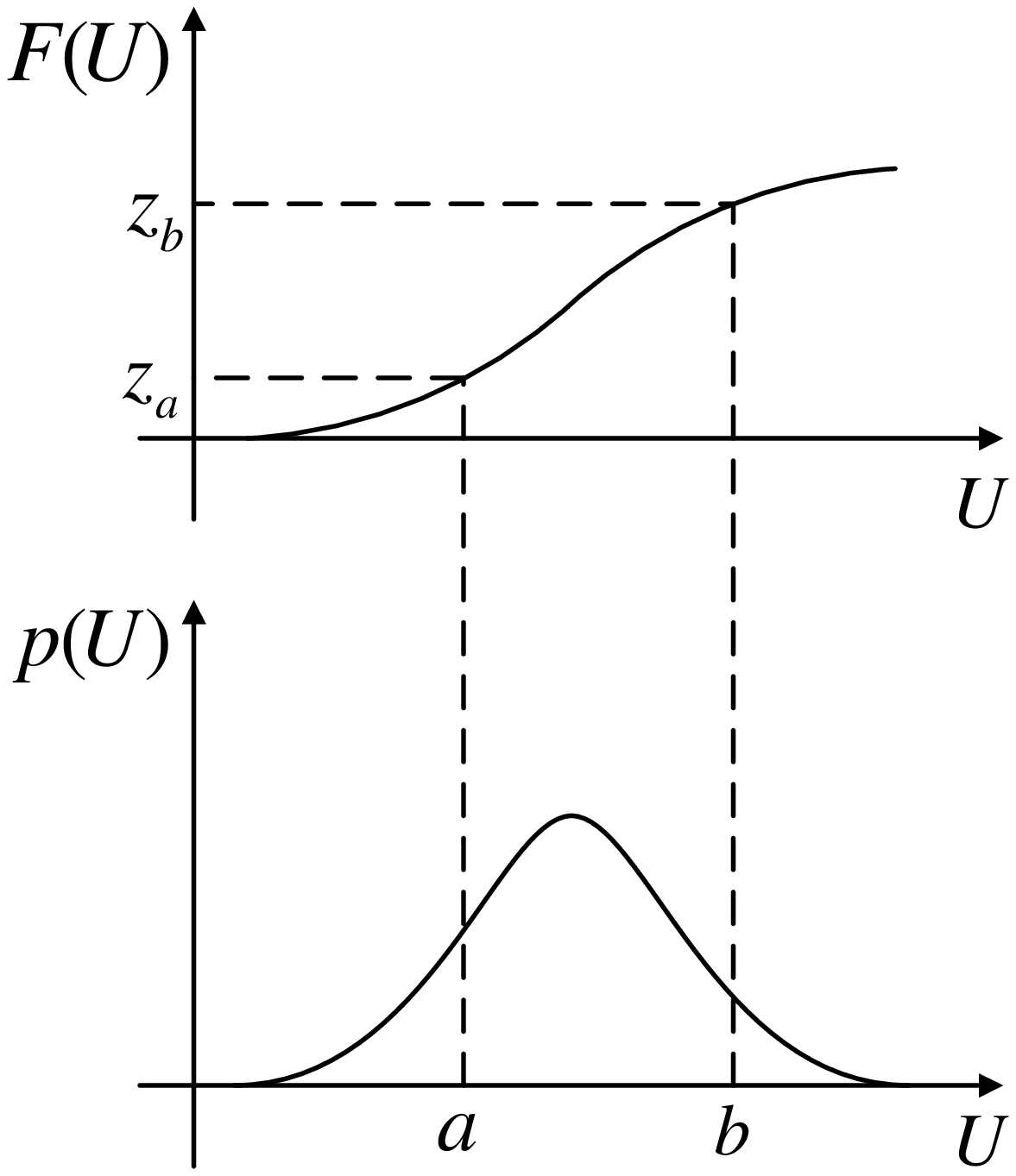}
\label{effcient_coding:sub2}
}
\caption{Fig. (a) shows an pipeline of linear and nonlinear transformation; Fig. (b)
shows an example of the cumulated distribution function (CDF) transformation, which can map
arbitrary distribution into a uniform distribution in $[0,1]$. }
\label{efficient_coding}
\end{figure}
To guide the structure learning for deep neural networks, we borrow a principle from computational neuroscience. In fact, various hypotheses have been proposed to understand the magic structure of biological nervous system in the literature. The core problem under investigation include: what is the goal of sensory coding or what type of neuronal representation is optimal in an ecological sense? Of all the attempts on answering this question, the principles rooted in information theory have been proved to be successful.

For ease of illustration of these principles, we give some notations first. Figure \ref{efficient_coding:sub1} shows how the data go through one layer of a typical neural network.
The input $\mathbf{X}$ is firstly processed by a linear transformation $\mathbf{W}$ and followed by a
component-wise transformation $\mathbf{\sigma}$, in which each component is usually a bounded,
invertible nonlinear function. That is,
\begin{equation}
Z_{i}=\sigma_{i}(U_{i})=\sigma_{i}(\mathbf{W}_{i}^T{\mathbf{X}})
\end{equation}
Without loss of generality, the range of $\sigma_{i}$ is usually assumed $[0,1]$. $\mathbf{Z}$
is the neuronal representation or a coding of the input $\mathbf{X}$.

Among the information-theoretic principles proposed in the literature, the \emph{redundancy reduction principle} developed by Barlow \cite{barlow_1961} has been at the origin of many theoretical and experimental studies. Let $I(\mathbf{Z})$ denote the mutual information between output units. Barlow's theory proposes that, the output of each unit should be statistically independent from each other in an optimal neural coding. That is, the objective function is the minimization of $I(\mathbf{Z})$. Another principle developed by Linsker \cite{linsker_1988} advocates that the system should maximize the amount of information that the output conveys about the input signal, that is, maximizing $I(\mathbf{X};\mathbf{Z})$.
As shown in the following theorem, the aforementioned principles are actually equivalent to each other under certain conditions.
\begin{theorem}
\label{theorem:mmi_me}
Let the component-wise nonlinear transfer function $\mathbf{\sigma}_{i}$ be the cumulated distribution
function (CDF) of $U_{i}$, minimizing $I(\mathbf{Z})$ is equivalent to maximizing $I(\mathbf{X};\mathbf{Z})$.
\end{theorem}
A sketch proof is given in supplementary material. The theory indicates that for bounded output neural networks, minimizing the mutual information between outputs is equivalent to maximizing the mutual information between inputs and outputs. Due to this equivalence, we will not distinguish the two principles, and will uniformly refer to them as ``efficient coding principles''.

Since the efficient coding principle is rooted in computational neuroscience and information theory, one may wonder whether it really ensures optimal neural network structures from the machine learning perspective. Through the following theorem we show that this principle has a strong theoretical connection with pattern classification tasks.
\begin{theorem}
\cite{minsky_1961} With (conditional) independent features, linear classifier is optimal in the sense of minimum Bayesian error.
\end{theorem}
Actually, this theorems has its particular implication in the context of structure learning for deep neural networks. As we know, no matter how a deep neural network is structured, its top layer is always a \emph{softmax} linear classifier. Then according to the above theorem, if we can achieve the independency between the nodes in the top hidden layer by means of structure learning, then it will ensure (as a sufficient condition) that the \emph{softmax} linear classifier would be the Bayesian optimal classifier. In other words, there is no need to adopt more complicated non-linear classifiers at all. In this sense, this theoretical foundation also provides us a clear way to determine the depth of the deep neural networks: if after multiple layers of non-linear transformations (learned under the guidelines of the efficient coding principle) the hidden nodes become statistically independent of each other, then we should stop growing the depth of the neural networks since we have already been optimal in terms of the possibly best classification error we could ever get.

\begin{figure}
\centering
\includegraphics[width=0.45\textwidth]{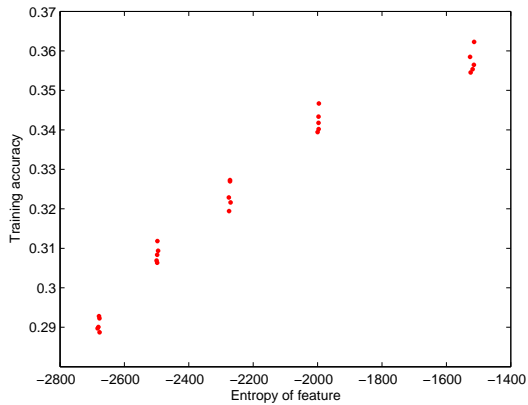}
\caption{The correlation between feature entropy and classification accuracy.}
\label{entropy_accuracy}
\end{figure}

\begin{figure*}
\centering
\subfigure[]{
\includegraphics[width=0.45\textwidth]{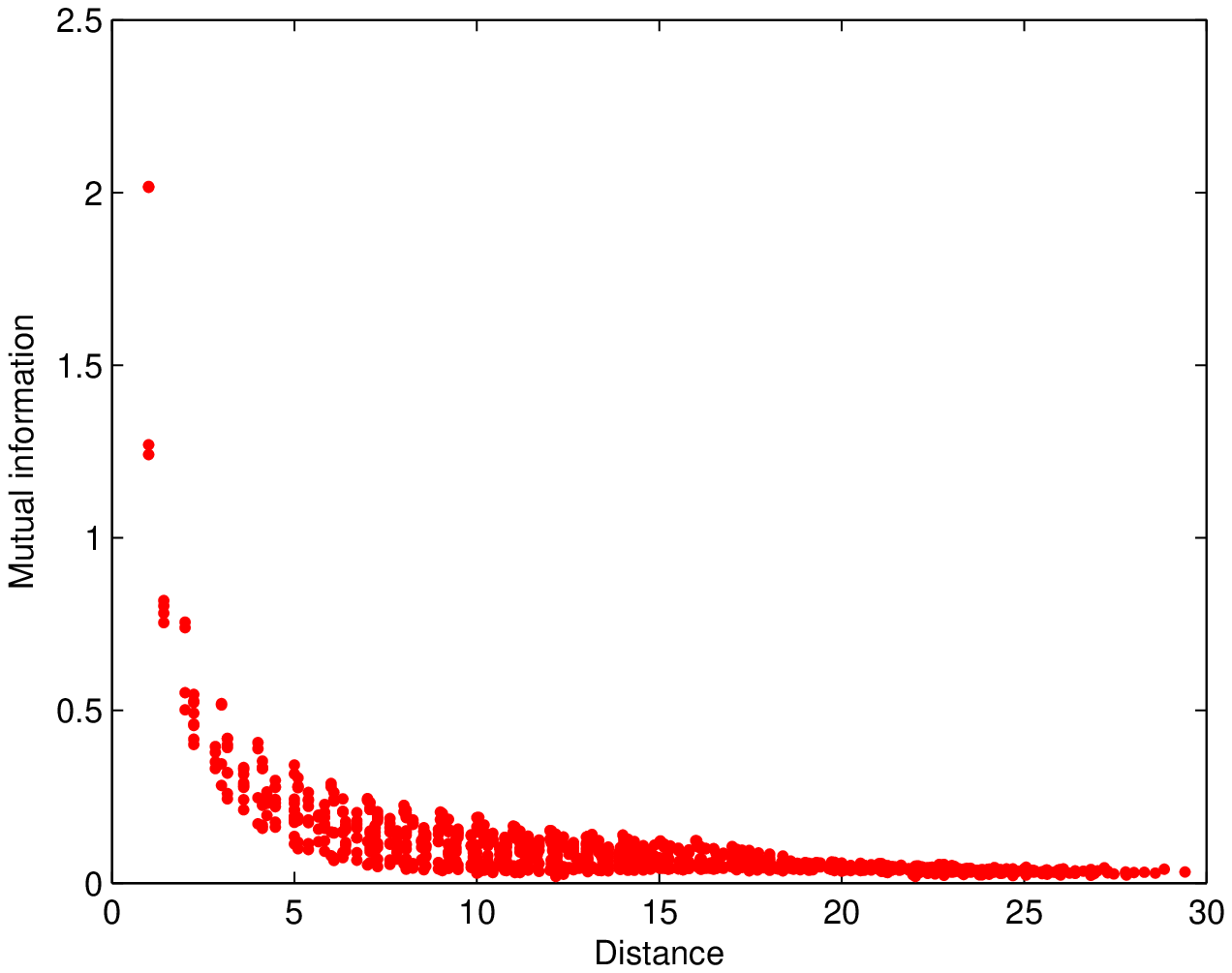}
\label{spatial_mi:pixel}
}
\hspace{0.03\textwidth}
\subfigure[]{
\includegraphics[width=0.45\textwidth]{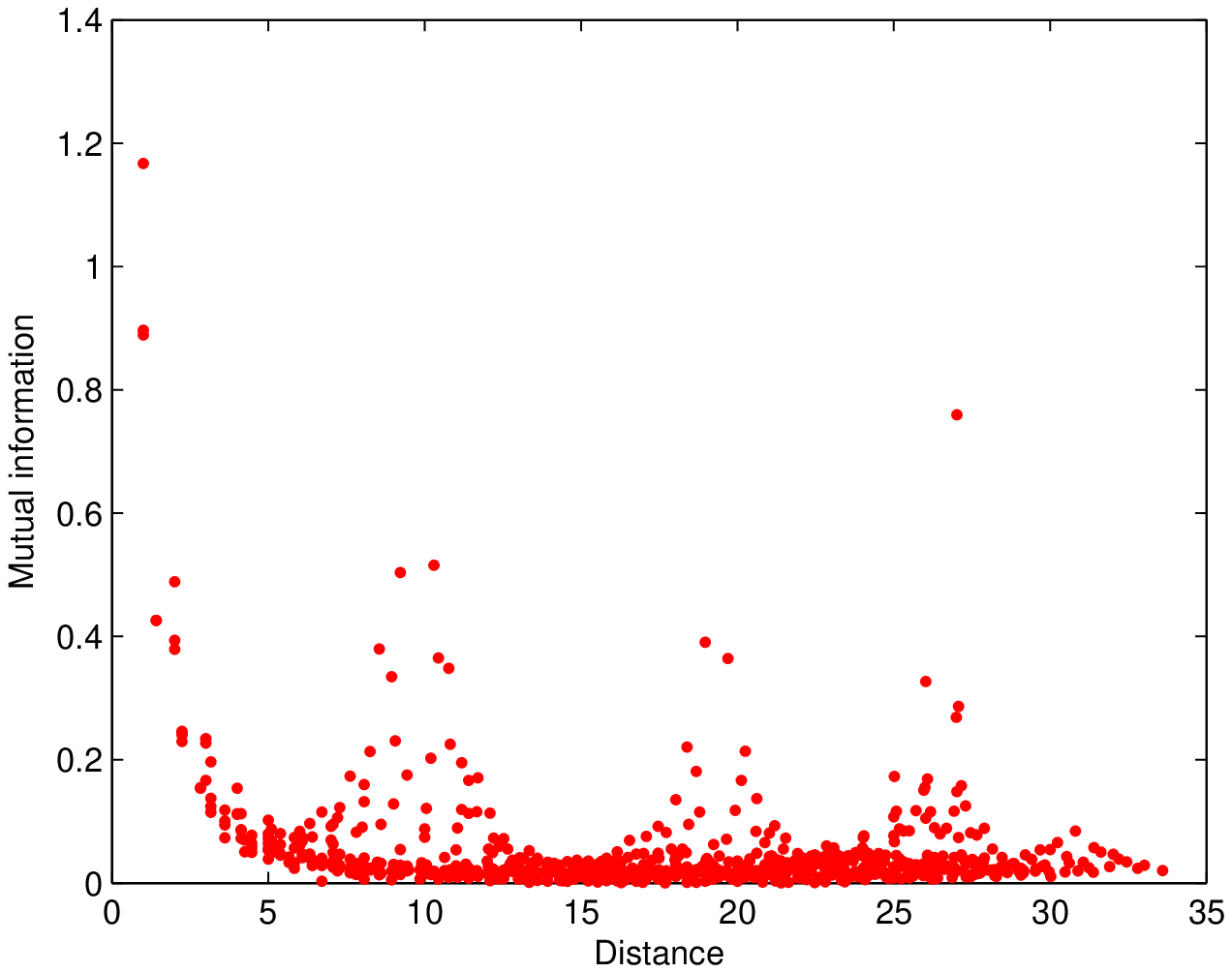}
\label{spatial_mi:edge}
}
\caption{Fig. (a) The mutual information between pixels decays with the increasing of the spatial distance between pixels; Fig. (b) The similar decay behavior can be observed
when each location is represented as an edge direction. This figure demonstrates that, with pixel representation, two locations seems independent if the spatial distance exceeds 10. However, after extracting
edge directions at each location and represent the location with the edge direction, two locations are no longer seemingly independent. Such observation implies the redundancies emerge with more abstract representation.}
\label{spatial_mi}
\end{figure*}
\section{Empirical Observations}
In this section, we show some empirical studies regarding the efficient coding principle. For this purpose, we need to compute the coding efficiency (or information gain) provided by a layer of neural network (in terms of the change of the mutual information). Please refer to the Appendix for the method we used.
%

\subsection{Coding Efficiency and Classification Accuracy}
\label{chapter:entropy_accuracy}
We generate a set of random structures and use them to extract features from whitened CIFAR-10 data set.
The coding efficiency of the extracted features and the corresponding classification on training set are
evaluated \footnote{We deliberately use training set because we want to know the relation between fitting quality and coding efficiency}.
Figure \ref{entropy_accuracy} demonstrates the relation between feature entropy and classification accuracy
of softmax classifier. A positive correlation between entropy and classification accuracy can be observed.
Noting that the efficient coding principle is an unsupervised objective function, this positive correlation is somewhat surprising.
%

\subsection{Spatial Redundancy of Images}
Figure \ref{spatial_mi} shows the redundancy properties of natural images. The result is obtained at
whitened CIFAR-10 data. The large value of mutual information between nearby elements indicates
the redundancy of information. In both figures, we can observe the decay behavior of mutual
information with the increasing of spatial distance. This suggests the elements sufficiently
far-away from each other are nearly independent. This is not surprising, as the phenomenon is
already widely formulated by Markov assumption in various probabilistic model.
Interestingly, Figure \ref{spatial_mi:edge} which shows that, after feature extraction by edge detector,
the redundancies between nearby pixels are removed, however, new dependencies among edges emerge.
The dependencies between edges spread much broader than that of pixels. This suggests that a single
layer transformation is not sufficient for the purpose of redundancy reduction.
We need another transformation to remove the
redundancies between edge representations.

\subsection{Multi-layer Redundancy Reduction}
Figure \ref{entropy_depth} illustrates entropy increasing after adding more layers of transformation. We train several layers of Gaussian-binary RBM on contrast normalized CIFAR-10 and
several layers of sparse coding on whitened CIFAR-10 \footnote{Gaussian-binary RBM assumes a Gaussian distribution of
the visual data, while sparse coding as an implementation of ICA assumes non-Gaussian data}. We can observe:
(1), For both RBM and sparse coding, an additional layer bring further entropy gain though the marginal gain vanishes
with the number of layers. (2), the sparse coding produce features with much higher entropy than
RBM because of sparse coding is more dedicated to the goal of redundancy reduction.

\begin{figure}
\centering
\includegraphics[width=0.45\textwidth]{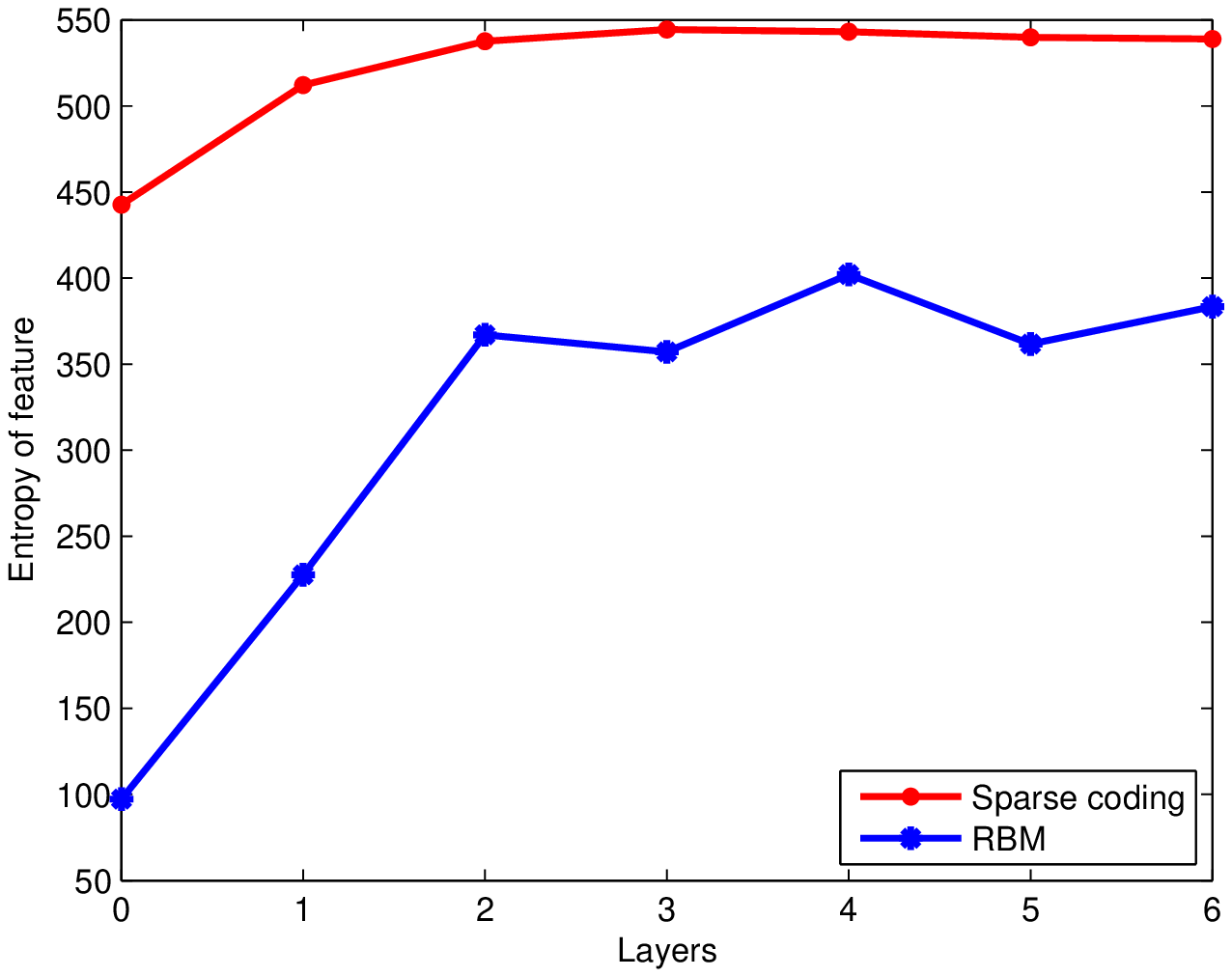}
\caption{With RBM and sparse coding, adding more layer always help remove
redundancy further. However, the gain of redundancy reduction gradually becomes marginal. }
\label{entropy_depth}
\end{figure}

\section{Algorithms}
\label{section:algorithm}
In this section, we propose algorithms for structure learning based on the principle of efficient coding.

\subsection{Implementing the Principle by Sparse Coding}
The transformation in Figure \ref{efficient_coding:sub1} shows how to get $\mathbf{U}$ from $\mathbf{X}$
through $\mathbf{U}=\mathbf{W}^{T}\mathbf{X}$. Inversely, $\mathbf{X}$ can be viewed as being generated
by a probabilistic process governed by $\mathbf{U}$, such as
\begin{equation}
\mathbf{X}=\mathbf{D}^{T}\mathbf{U}+\mathbf{N}
\end{equation}where $\mathbf{N}$ is a vector of independently and identically distributed Gaussian noise, $\mathbf{D}$ describes how each source signal
generates an observation vector. If the dimensions of $\mathbf{X}$ and $\mathbf{U}$ are equal and transformation matrix $\mathbf{W}$ is full rank,
a trivial relation between $\mathbf{V}$ and $\mathbf{W}$ is $\mathbf{W}^{T}\mathbf{D}=\mathbf{I}$, where $\mathbf{I}$ is an identify matrix. In this way,
learning optimal data transformation matrix $\mathbf{W}$ is equivalent to inferring the optimal data generating matrix $\mathbf{D}$.

Usually, to infer the simplest possible signal formation process, some additional assumptions on the prior
distribution of $\mathbf{U}$ are made. As the efficient coding principle indicates,  the independent or factorial codes are preferred, that is, assuming
\begin{equation}
p(\mathbf{U})=\prod_{i}p(U_{i}).
\end{equation}
When $p(U_{i})$ is a distribution peaked at zero with heavy tails, it can be shown the model leads to the so-called sparse coding \cite{field_1994}. Therefore our structure learning algorithm described in the next subsection is based on sparse coding\footnote{We note that both ICA and sparse coding can implement the efficient coding principle. The former adds assumption on the CDF of $U_{i}$
while the latter adds assumption on the probability density function of $U_{i}$. However, ICA is difficult to be extended to over-complete case, and we do not use it in this work.
}.

\begin{figure}
\centering
\subfigure[]{
\includegraphics[width=0.15\textwidth]{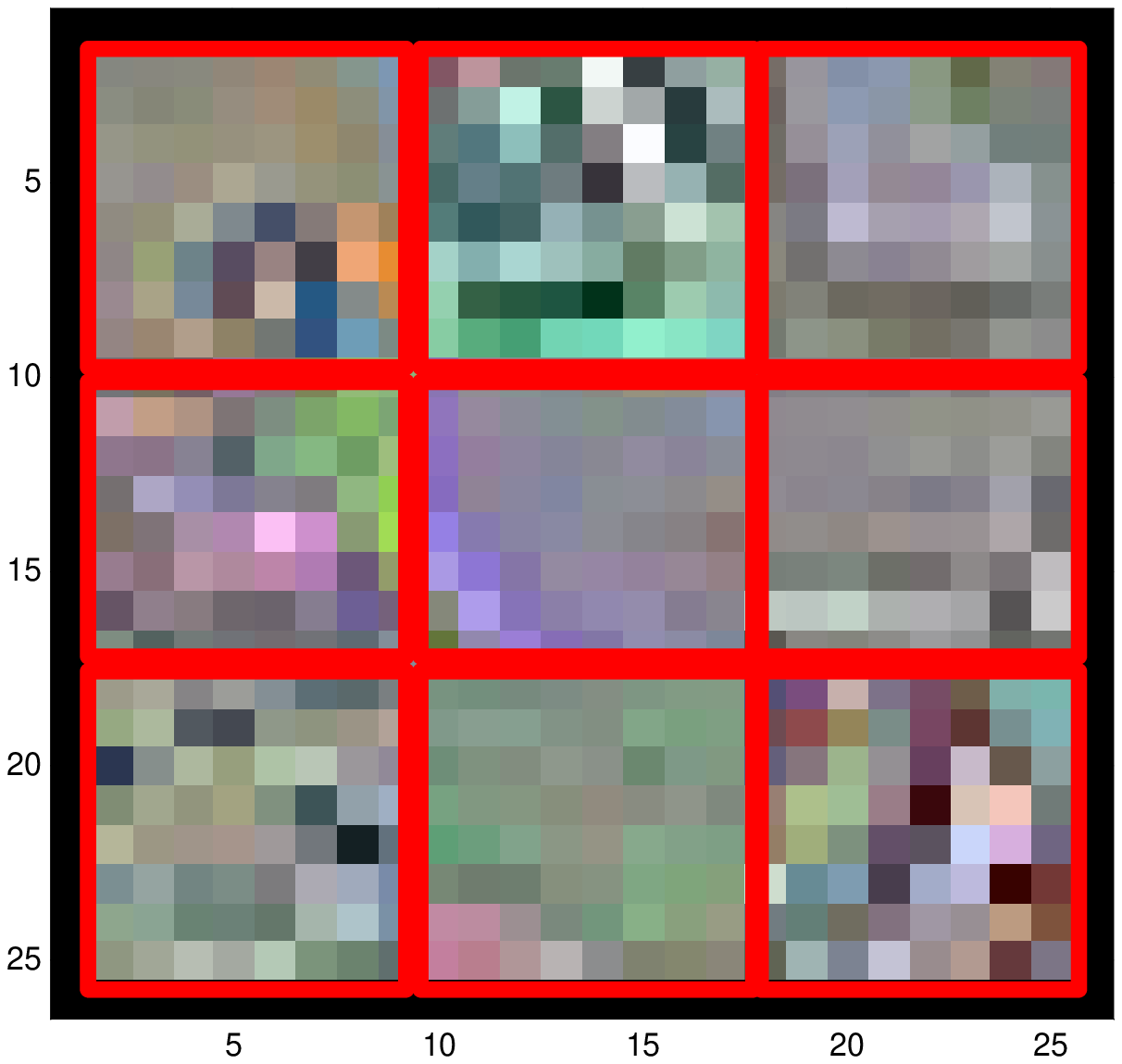}
\label{block_toy:sub1}
}
\subfigure[]{
\includegraphics[width=0.25\textwidth]{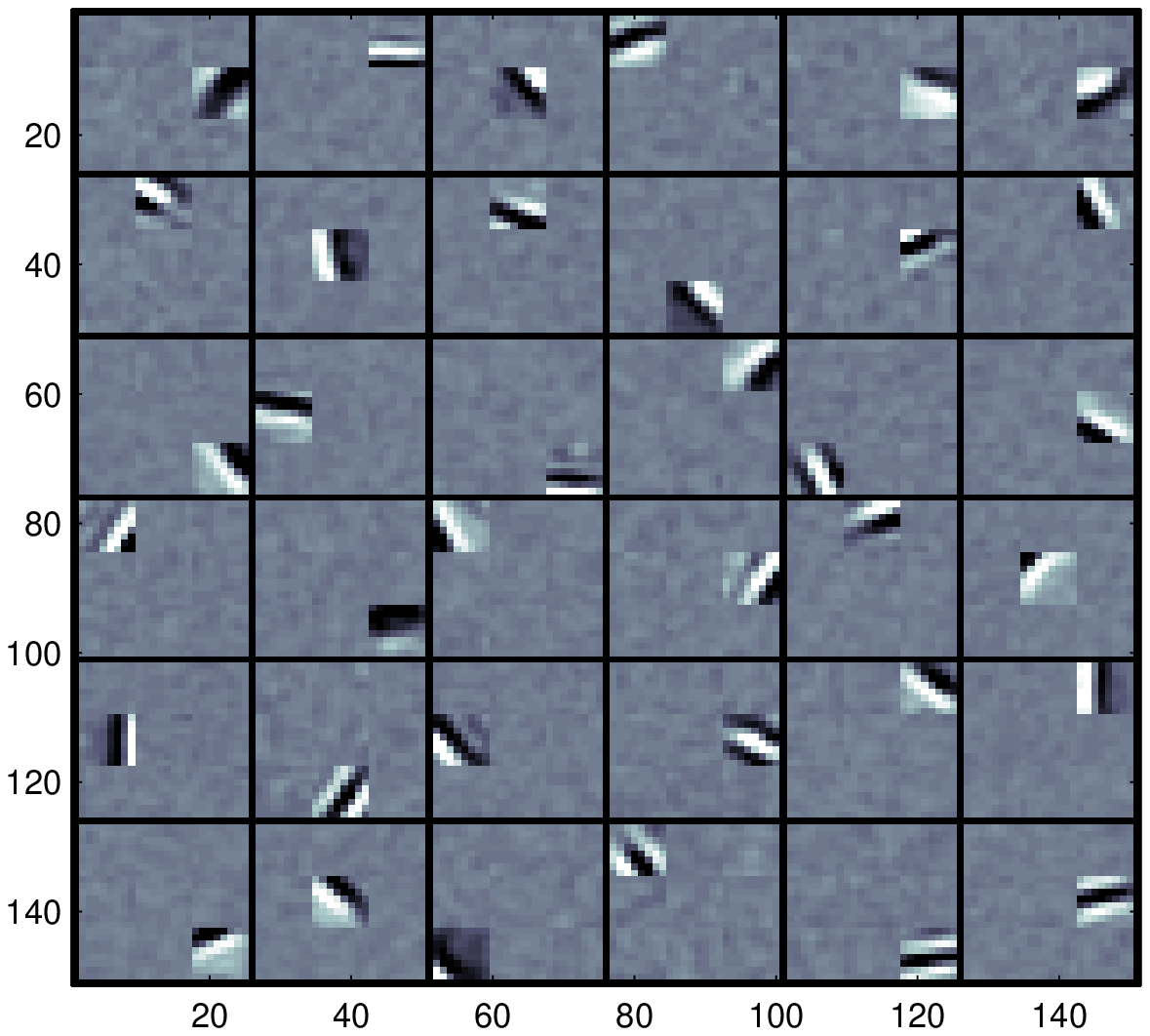}
\label{block_toy:sub2}
}
\caption{(a) An example of synthesize images: each image is composed of 9 blocks, and
each block is sampled from a random patch of a random image. (b) 36 basis
learned on the synthesized image set by sparse coding. Apparently, each basis only focuses
on one block of the input image. }
\label{block_toy}
\end{figure}

We construct a synthesized data set from whitened CIFAR-10. Each image consists of $3 \times 3$ blocks,
each of which is sampled from a random location of a random image. We can assume
the pixels from different blocks are independent from each other, while the pixels in the same block
possess the statistical properties of natural images. An ideal algorithm should discover
and converge to the correct structure. That is, each basis is only connected to pixels in a single
block. We evaluated RBM, sparse auto-encoder, and sparse coding. Sparse coding (see Fig. \ref{block_toy}) is the only one that can perfectly recover the block structures.

\subsection{The Proposed Algorithms}

Algorithm \ref{alg:structure_learning} shows how to learn the structure from unsupervised data.
As shown in the algorithm, we learn the structure layer by layer in a bottom-up manner.
We put the raw features at layer $1$ and learn the connection between layer $1$ and layer $2$
given the predefined number\footnote{As aforementioned, we focus on inter-layer connection
learning and depth learning in this work and assume the number of nodes of each layer is given
in advance. One can also leverage the technique proposed in \cite{chen_2013} to automatically
learning the number of nodes in each layer.} of nodes in layer $2$. Specifically, the inter-layer
connection is initialized with full connections. Trained on unlabeled data, due to the ICA
properties of sparse coding, the inter-layer connection will converge to a sparse connected
network. In other words, the weights of most of the edges will converge to zero during the
learning process. After the connection between layer $k$ and $k+1$ has been learnt, we can
estimate the entropy of layer $k+1$ according to Equation \ref{equation:ee_knn} and compare
it with that of layer $k$. If the entropy gain between the two layers is smaller than
a threshold $\epsilon$, we terminate the learning process; otherwise we add layer $k+2$ and
continue to learning process. In other words, the depth of the network is determined according to the cues of entropy gain.

Algorithm \ref{alg:bp_structure_priming} shows how to learn a better DNN based on the structure
output by Algorithm \ref{alg:structure_learning} and supervised data.  First, the learned
sparse connection and the weights are used to initialize a multi-layer feed-forward network.
The inter-layer connection inherits from the structure mask learned in Algorithm \ref{alg:structure_learning},
and the weights of connection is initialized by the weights learned by sparse coding.
Then via training on labeled data, the weights\footnote{The structure will be fixed} are fine-tuned further.

While the two algorithms look quite similar to several existing algorithms at the first glance,
we would like to highlight and elaborate several differences and some implementation details as follows.
\begin{algorithm}[tb]
   \caption{Structure Learning with Sparse Coding}
   \label{alg:structure_learning}
\begin{algorithmic}[1]
   \STATE {\bfseries Input:} $\mathbf{X}$, each column is an example $\mathbf{X}_{i}\in \mathcal{R}^{d}$
   \STATE {\bfseries Output:} dictionary $\mathbf{D}_{k}$, structure mask $\mathbf{M}_{k}$, depth $K$
   \STATE Initialize $k=0$, $\mathbf{U}_{k}=\mathbf{X}$
   \REPEAT
   \STATE Whiten $\mathbf{U}_{k}$
   \STATE $\mathbf{D}_{k},\!\mathbf{U}_{k+1}\!\!=\!\!\underset{\mathbf{D}_{k},\mathbf{U}_{k+1}}{\argmin} \|\mathbf{U}_{k}-\mathbf{D}_{k}^{T}\mathbf{U}_{k+1}\|_{2}\!+\!\lambda \|\mathbf{U}_{k+1}\|_{1}$
   \STATE Calculate $\mathbf{M}_{k}$ by thresholding $\mathbf{D}_{k}$
   \STATE $\mathbf{Z}_{k+1}=\sigma(\mathbf{U}_{k+1})$, where $\sigma$ is feature-wise CDF
   \STATE Estimate $H(\mathbf{Z}_{k+1})$
   \STATE $k=k+1$
   \UNTIL{$|\frac{H(\mathbf{Z}_{k+1})-H(\mathbf{Z}_{k})}{H(\mathbf{Z}_{k})}| < \epsilon$}
   \STATE $K = k+1$
\end{algorithmic}
\end{algorithm}
\begin{algorithm}[tb]
   \caption{Back-propagation with Structure Priming}
   \label{alg:bp_structure_priming}
\begin{algorithmic}[1]
   \STATE {\bfseries Input:} $\mathbf{X}$, $\mathbf{D}_{k}$, $\mathbf{M}_{k}$, $k=1,\ldots,K$
   \STATE {\bfseries Output:} $\mathbf{W}_{k}$, $k=1,\ldots,K$
   \STATE $\mathbf{U}_{0}=\mathbf{X}$
   \STATE $\mathbf{W}_{k}=\mathbf{M}_{k}\circ\mathbf{D}_{k}$ for $k=1,\cdots,K$
   \STATE /* Forward pass */
   \FOR {$k=0$ {\bfseries to} $K$}
   \STATE $\mathbf{U}_{k+1}=\mathbf{W}_{k}^{T}\mathbf{U}_{k}$
   \STATE $[\mathbf{U}_{k+1}]_{i}=\mathrm{sign}([\mathbf{U}_{k+1}]_{i})(|[\mathbf{U}_{k+1}]_{i}|-\lambda)_{+}$
   \ENDFOR
   \STATE /* Back-propagation */
   \STATE $\delta \mathbf{U}_{K}=\frac{\partial \mathrm{Loss}}{\partial \mathbf{U}_{K}}$
   \FOR {$k=K$ {\bfseries down to} $1$}
   \STATE $[\delta \mathbf{U}_{k}]_{i} = 0$ if $[\mathbf{U}_{k}]_{i} = 0$
   \STATE $\delta \mathbf{U}_{k-1} = \mathbf{W}_{k-1}\mathbf{U}_{k}$
   \STATE $\delta \mathbf{W}_{k-1} = \mathbf{U}_{k-1}\delta\mathbf{U}_{k}^{T}$
   \STATE $\delta \mathbf{W}_{k-1} = \mathbf{M}_{k-1} \circ \delta\mathbf{W}_{k-1}$
   \STATE $\mathbf{W}_{k-1} = \mathbf{W}_{k-1} - \gamma \delta \mathbf{W}_{k-1}$
   \ENDFOR
\end{algorithmic}
\end{algorithm}
\paragraph{Training sparse coding in global range} The dramatic difference between the usage
of sparse coding in this paper and that of existing work is that, the sparse coding is trained
in global range. Take image as an example, here we train sparse coding on the whole image instead
of traditional way which trains sparse coding on small patches \cite{yang_2009}. Therefore,
we do not need to predefine the inter-layer connection such as the spatial range of local
connection in convolutional networks. Instead, the algorithm itself is able to learn the
optimal inter-layer connections. We will see in Section \ref{section:experiments},
the inter-layer connections learned on images happens to resemble the local
connection structure in CNN. We notice that deconvolutional network proposed
in \cite{zeiler_2010} also trains sparse coding on whole image; however, like
the patch-based sparse coding, it also needs a pre-determined spatial range for
convolutional filter.
\paragraph{Sparsifying inter-layer connection} Once obtaining $\mathbf{D}_{k}$
by Algorithm \ref{alg:structure_learning}, we know the strength of each connection
between adjacent-layer neurons.
Intuitively, the weak connection can be removed without significantly affect
the behavior of network. Concretely, a binary mask matrix $\mathbf{M}_{k}$ is calculated by thresholding $\mathbf{D}_{k}$
\begin{equation}
[\mathbf{M}_{k}]_{i}=\left\{\begin{array}{rl}
1, & |[\mathbf{D}_{k}]_i| \geq t\\
0, & \mathrm{otherwise}
\end{array}
\right.
\end{equation}
The parameter $t$ can be chosen according to what density of the mask matrix
$\mathbf{M}_{k}$ we expect. For example, if we want to keep $10\%$ connections
in the network, we can calculate the histogram of the absolute value in $\mathbf{D}_{k}$
and choose $t$ at the position of $10\%$ quantile. The resulting $\mathbf{M}_{k}$
will act as a structure priming for the feed-forward network in Algorithm \ref{alg:bp_structure_priming}.

\paragraph{Handling invariance} In tasks such as image classification, invariance
based on pooling is crucial for practical performance. The output of neurons
who have similar receptive fields are aggregated
through an OR operation to obtain shift invariance. To endow the algorithm with
this capability, the neurons should firstly be separated into overlapping or
non-overlapping groups according to their selectivities. Since Algorithm \ref{alg:structure_learning}
is based the sparse coding framework, it can be easily modified to handle
invariance using group lasso\footnote{ Similar ideas appear in \cite{hyvarinen_tica_2001,le_icml_2012}.} \cite{yuan_2006}. With group lasso, the dictionary learning becomes
\begin{equation}
\mathbf{D}_{k}=\underset{\mathbf{D}_{k}}{\argmin} \|\mathbf{U}_{k}\!-\!\mathbf{D}_{k}^{T}\mathbf{A}\|_{2}\!+\!\lambda \sum_{n}\!\sum_{g} \|\mathbf{A}_{g}^{n}\|_{2}
\end{equation}
where $\mathbf{A}_{g}^{n}$ denotes the reconstruction coefficients of the $n$-th
example in the $g$-th group. Also, the shrinkage operation in FISTA \cite{beck_2009} algorithm will change accordingly.


\paragraph{Back-propagation with structure priming} Both $\mathbf{M}_{k}$ and $\mathbf{D}_{k}$
provide structure priming for feed-forward network. The weights $\mathbf{W}_{k}$ are initialized as $\mathbf{M}_{k}\circ\mathbf{D}_{k}$ in feed-forward pass, and $\mathbf{M}_{k}$ also mask the gradient $\delta \mathbf{W}_{k-1}$ in back-propagation by
\begin{equation}
\delta \mathbf{W}_{k-1} = \mathbf{M}_{k-1} \circ \mathbf{W}_{k-1}
\end{equation}
where $\circ$ denotes Hadamard product. This implementation is the same as DropConnect approach in
\cite{wan_2013}. However, in DropConnect, the mask matrix is randomly generated and also it only used
in full connection layer.
\paragraph{One step ISTA approximation} Besides using the dictionary $\mathbf{D}_{k}$ to initialize
$\mathbf{W}_{k}$, we also would like the feature map to inherit the sparse properties from sparse coding.
Therefore, we implement the feed-forward transformation as a one-step ISTA approximation to the
solution of lasso. The idea comes from \cite{gregor-2010} which uses several steps to obtain an
efficient approximation to lasso. The shrinkage operation for standard sparse coding and group sparse coding are
\setlength{\arraycolsep}{2pt}
\begin{eqnarray}
[\mathbf{U}_{k+1}]_{i}&=&\mathrm{sign}([\mathbf{U}_{k+1}]_{i})(|[\mathbf{U}_{k+1}]_{i}|-\lambda)_{+}\\
{}[\mathbf{U}_{k+1}]_{i}&=&\mathrm{sign}([\mathbf{U}_{k+1}]_{i})(1-\frac{\lambda}{\|[\mathbf{U}_{k+1}]_{g}^{i}\|_{2}})_{+}
\end{eqnarray}
\setlength{\arraycolsep}{5pt}where $\|[\mathbf{U}_{k+1}]_{g}^{i}\|_{2}$ indicates the $\ell$-2
norm of the group variables $[\mathbf{U}_{k+1}]_{i}$ belongs to. Note that the nonlinear transfer
function used in feed-forward network is different from the transfer function used for estimating
entropy $H(\mathbf{Z}_{k})$ where the transfer function the CDF of feature map.

\paragraph{CUDA implementation of sparse coding} Training a sparse coding model on whole images
instead of patches is much demanding to computation resource. In our experiments, even the highly
optimized sparse modeling package SPAMS \cite{marial_2010} requires several days to convergence.
Regarding that GPGPU has been a common option for deep learning algorithm, we implement a CUDA
sparse coding based on FISTA algorithm. As noted in \cite{gregor-2010}, coordinate descent (CD)
may be the fastest algorithm for sparse inference, it is true for CPU implementation. However,
CD algorithm is not innate to parallel implementation due to its sequential nature.
Our CUDA implementation of an online sparse coding algorithm based on FISTA algorithm speedup the process $7\!\sim\!10$ times over SPAMS.

\section{Experiments}
\label{section:experiments}
\begin{table}[t]
\caption{Classification accuracies of baseline convolutional neural network.}
\label{baseline}
\vskip 0.15in
\begin{center}
\begin{small}
\begin{sc}
\begin{tabular}{lcr}
\hline
\abovespace\belowspace
ID & Configuration  & Test accuracy \\
\hline
\abovespace
cudaconv & Adaptive LR              & $81.0\%$ \\
standard  & Fixed LR                & $77.5\%$ \\
nodropout & - dropout               & $75.7\%$\\
nopadding & - padding              & $73.3\%$ \\
nonorm    &  - normalization          & $73.1\%$\\
nopooling & - pooling &               $65.4\%$\\
2conv     & nonorm with 2 conv            & $72.5\%$\\
1conv     & nonorm with 1 conv            & $60.4\%$\\
\hline
\end{tabular}
\end{sc}
\end{small}
\end{center}
\vskip -0.1in
\end{table}
\subsection{Baseline and experimental protocols}
All the experiments are carried out on the well-known CIFAR-10 data set. It contain 10 classes
and totally has 50000 and 10000 $32 \times 32$ color images in training and testing set respectively.
As a baseline, we reproduce the cuda-convnet experiments on CIFAR-10 \cite{krizhevsky_2012}.
In our implementation, we don't use the data augmentation technique such as image translation and horizontal reflection,
except that the mean image is subtracted from all the images as the cuda-convnet setting does.
With an adaptive learning rate, we get an $81\%$ test accuracy with a single model. Since we are
focusing on the network structure such as inter-layer connection and depth, we evaluate the contribution
of techniques such as adaptive learning rate, dropout, padding and local normalization to the baseline
system. These results are reported in Table \ref{baseline}. Note that in the table, the configuration
in each line is modified from the above line without otherwise stated. Therefore, the NONORM is a setting
without modules such as adaptive learning
rate, dropout, padding and local normalization. This setting reflects the contribution of network architecture.
If we further remove the pooling layer from NONORM, the performance drops to $65.4\%$, which implies pooling
plays a significant role in this task. Removing the $3$-rd convolutional layer from NONORM, the performance
drops slightly to $72.5\%$, which implies the $3$-rd convolution layer contributes marginally.
If further removing the $2$-nd convolutional layer from 2CONV, the accuracy drops dramatically to $60.4\%$,
which indicates two convolutional layers are essential.

Different from the input of convolutional network, the input images to sparse coding are all whitened.
Empirically, we find whitening is crucial for sparse coding to get meaningful structure \footnote{This is
consistent with the observation that whitening is essential for independent component analysis (ICA) \cite{hyvarinen_2001}}.
In both CNN and our algorithm, the learning rate is fixed to $0.001$. Without otherwise stated,
the network includes a 10-output softmax layer. The sparse coding dictionaries in all the layers
are with 2048 dimension and group size is 4.
All the experiments are carried out on a Tesla T20K GPU.
\subsection{Overall performance}
\begin{table}[t]
\caption{Classification accuracies of structure learning.}
\label{structure}
\vskip 0.15in
\begin{center}
\begin{small}
\begin{sc}
\begin{tabular}{lcr}
\hline
\abovespace\belowspace
ID & Configuration  & Test accuracy \\
\hline
\abovespace
1layer & 512                   & $63.0\%$ \\
2layer  & 512/512                & $68.0\%$ \\
3layer & 512/512/512                & $69.8\%$\\
\hline
\end{tabular}
\end{sc}
\end{small}
\end{center}
\vskip -0.1in
\end{table}
From Table \ref{structure}, we can observe that the single layer network of learned structure
achieves a test accuracy $63.0\%$ which outperforms $60.4\%$ of the single layer setting in Table \ref{baseline}.
The two layer architecture achieves a performance $68\%$ which outperforms one layer model but
is inferior to $72.5\%$ produced by CNN with two convolutional layer.

\subsection{Evaluating inter-layer connection density}
\begin{table}[t]
\caption{Evaluating the role of inter-layer connections.}
\label{table:density}
\vskip 0.15in
\begin{center}
\begin{small}
\begin{sc}
\begin{tabular}{lccr}
\hline
\abovespace\belowspace
Density & random  & RBM & sparse coding\\
\hline
\abovespace
$0.1\%$         &  $30.2\%$      & $23.2\%$ & $40.7\%$\\
$0.25\%$         &  $31.6\%$      & $28.4\%$ & $40.1\%$\\
$0.5\%$         &  $30.3\%$      & $31.3\%$ & $51.2\%$\\
$10\%$         &  $48.2\%$      & $39.7\%$ & $57.2\%$\\
$30\%$         &  $54.7\%$      & $39.6\%$ & $56.2\%$\\
$70\%$        &  $56.2\%$      & $39.4\%$ & $56.6\%$\\
\hline
\end{tabular}
\end{sc}
\end{small}
\end{center}
\vskip -0.1in
\end{table}
To demonstrate the role of inter-layer connection, we compare three types of structures in a single layer network: (1) randomly generated structures,
(2) structures by sparsifying restricted Boltzmann machines (RBM), (3) structures learned by sparse coding.
We define the connection density as the ratio between
active connections in structure mask $\mathbf{M}_{k}$ and the number of full connections.
As shown in Table \ref{table:density} we evaluate the three settings in several connection
density level. Basically, we can observe denser connections bring performance gain.
However, the performance saturates even keeping randomly $30\%$ connections.
The structures generated by sparse coding outperform the random structure. Surprisingly, structures generated by RBM are even inferior to random structures.
In the learned structure, at the same density level, sparse coding always outperforms RBM.
For sparse coding, by keep $10\%$ connections is sufficient.
Theoretically, a full connection network can emulate any sparse connection ones just by setting the dis-connected weights of the
to zeros. If a sparsely connected network is known to be optimal, ideally,
a full connection network with appropriate weights can yield exactly the same behavior.
However, the BP algorithm usually can not
converge to this optimal weights, due to the local optimum properties.

\subsection{Evaluating the role of structure mask}
\begin{table}[t]
\caption{Evaluating the role of structure mask.}
\label{table:mask}
\vskip 0.15in
\begin{center}
\begin{small}
\begin{sc}
\begin{tabular}{lcr}
\hline
\abovespace\belowspace
Setting & One layer  & Two layer\\
\hline
\abovespace
BP           &  $57.9\%$      & $50.2\%$ \\
Weight       &  $57.6\%$      & $57.4\%$ \\
Weight+BP    &  $63.7\%$      & $62.4\%$ \\
Mask+BP       &  $58.8\%$      & $58.1\%$ \\
Weight+Mask+BP &  $64.0\%$      & $68.6\%$\\
\hline
\end{tabular}
\end{sc}
\end{small}
\end{center}
\vskip -0.1in
\end{table}
To investigate the role of the learned structure mask, we carry out the following experiments: (1) randomly initialized BP;
(2), initializing the network parameter without fine-tuning; (3), initializing the
network parameter with pre-trained dictionary and fine-tuned with BP; (4), restricting
the network structure with learned mask and randomly initializing the parameters, finally fine-tuning with BP;
(5), restricting the network structure with mask and initializing
the connection parameters with pre-trained dictionary, finally fine-tuning with BP.
The results are reported in Table \ref{table:mask}. We can observe that, even use the pre-trained
dictionary as a feature extractor, it significantly outperforms BP with random initialization.
Fine-tuning with BP always brings performance gain. The Mask+BP outperforming BP indicates that
the structure prior provided by sparse coding is very useful. Finally, the strategy of combining
Weight+Mask+BP outperforms all the others.

\subsection{Evaluating network depth}
Table \ref{structure} shows the performances of networks with different depth. We can observe that adding more layers is helpful,
however the marginal performance gain diminishes as the depth increases. Interestingly, Figure \ref{entropy_accuracy} shows
a similar behavior of coding efficiency. This empirically justifies the approach determining the depth by coding efficiency.

\section{Conclusions}
In this work, we have studied the problem of structure learning for DNN. We have proposed to use the principle of efficient coding principle for unsupervised structure learning and have designed effective structure learning algorithms based on global sparse coding, which can achieve the performance as good as the best human-designed structure (i.e., convolutional neural networks).

For the future work, we will investigate the following aspects. First, we have empirically shown that redundancy reduction is positively correlated to accuracy improvement. We will explore the theoretical connections between the principle of redundancy reduction and the performance of DNN. Second, we will extend and apply the proposed algorithms to other applications including speech recognition and natural language processing. Third, we will study the structure learning problem in the supervised setting.

\bibliography{structure}
\bibliographystyle{icml2014}


\section{Appendix}


\begin{theorem}
\label{theorem:mmi_me}
Let the component-wise nonlinear transfer function $\mathbf{\sigma}_{i}$ be the cumulated distribution
function (CDF) of $U_{i}$, minimizing $I(\mathbf{Z})$ is equivalent to maximizing $I(\mathbf{X};\mathbf{Z})$.
\end{theorem}

\section{The Proof Sketch of Theorem \ref{theorem:mmi_me}}
Under the conditions of no noise \cite{bell_1995} or only additive output noise \cite{nadal_1994},
the term $H(\mathbf{Z}|\mathbf{X})$ has no
relation to the transformation $\mathbf{W}$ and $\mathbf{\sigma}$. Therefore, we have
\begin{lemma}
\label{lemma:nadal}
\cite{nadal_1994} Maximizing the mutual information $I(\mathbf{X};\mathbf{Z})$ between
input and output is equivalent to maximizing the output entropy $H(\mathbf{Z})$
under no noise or only additive output noise.
\end{lemma}
Furthermore, as shown in Figure \ref{effcient_coding:sub2}, when adopting each
bounded and invertible transfer function $\sigma_{i}$ as the cumulated distribution
function (CDF) of
the linear output $U_{i}$, each $Z_{i}$ follows a uniform distribution in $[0,1]$
and thus $H(Z_{i})$ achieves its maximum value. In this way, we have
\begin{lemma}
\label{lemma:yang}
\cite{yang_1997} Maximizing $H(\mathbf{Z})$ is equivalent to minimizing $I(\mathbf{Z})$
when the component transfer function $\sigma_{i}$ is the CDF of $U_{i}$.
\end{lemma}
Now we can prove Theorem \ref{theorem:mmi_me}, which states the equivalence between
the maximum information preservation principle of \cite{linsker_1988} and the the
redundancy reduction principle of \cite{barlow_1961}.
\begin{proof}
Firstly, we have the following basic facts,
\setlength{\arraycolsep}{3pt}
\begin{eqnarray}
I(\mathbf{X};\mathbf{Z})&=&H(\mathbf{Z}) - H(\mathbf{Z}|\mathbf{X})\\
H(\mathbf{Z})&=&\sum_{i}H(Z_{i})-I(\mathbf{Z})
\end{eqnarray}
\setlength{\arraycolsep}{5pt}Using the conclusion of Lemma \ref{lemma:nadal} and
Lemma \ref{lemma:yang}, we know that maximizing $I(\mathbf{X};\mathbf{Z})$ is
equivalent to minimizing $I(\mathbf{Z})$ under the given conditions.
\end{proof}

\section{How to Measure the Efficiency Gain}
\label{appendix:gain}
\begin{figure}
\centering
\includegraphics[bb = 63 87 576 346,width=0.3\textwidth]{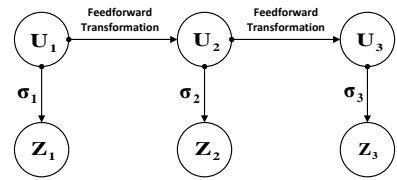}
\caption{Measuring high-dimensional mutual information via CDF transfer function}
\label{mlp_entropy}
\end{figure}
To determine the depth of a feed-forward neural networks, we need to know at what
depth the additional layer will not bring efficiency gain.
Remember that the efficiency means fewer redundancies among all dimensions, therefore,
we need to measure the independence between all all features in the same layer.
As shown in Figure \ref{mlp_entropy}, let $\mathbf{U}_{i}, i=1,2,3$ be three feature
maps in a feed-forward networks, direct computing $I(\mathbf{U}_{1})$, $I(\mathbf{U}_{2})$ and $I(\mathbf{U}_{3})$ are infeasible
for high dimensional data. This paper develops a technique to quantitatively measure
those criteria. (For simplicity of analysis, we assume the dimensions of all the layer are the same).
\begin{lemma}
If $\sigma_{i}$ is the CDF of $U_{i}$, $I(\mathbf{U})=-H(\mathbf{Z})$ holds.
\label{lemma:iu_hz}
\end{lemma}
\begin{proof}
Since the CDF transfer function will make each output $Z_{i}$ with a uniform distribution
in $[0,1]$, $H(Z_{i})$ will be zero. Therefore, $I(\mathbf{Z})=-H(\mathbf{Z})$.
The lemma \ref{lemma:iu_iz} states that $I(\mathbf{U})=I(\mathbf{Z})$ if $\sigma_{i}$ is
continuous and invertible. Apparently, the CDF transfer function satisfies this condition. Thus, $I(\mathbf{U})=-H(\mathbf{Z})$.
\end{proof}
\begin{lemma}
If the component-wise transfer function $\sigma_{i}$ in $Z_{i} = \sigma_{i}(U_{i})$ is
continuous and invertible, $I(\mathbf{U})$ equals $I(\mathbf{Z})$.
\label{lemma:iu_iz}
\end{lemma}
\begin{proof}
Firstly, we know the facts
\setlength{\arraycolsep}{3pt}
\begin{eqnarray}
I(\mathbf{Z})&=&\sum_{i}H(Z_{i}) - H(\mathbf{Z})\label{component:mi_z}\\
I(\mathbf{U})&=&\sum_{i}H(U_{i}) - H(\mathbf{U})\\
H(Z_{i})&=&H(U_{i}) + \int_{U_{i}} \mathrm{d}U_{i} p(U_{i}) \log |\frac{\partial Z_{i}}{\partial U_{i}}|\\
H(\mathbf{Z})&=&H(\mathbf{U})+\int_{\mathbf{U}}\mathrm{d}\mathbf{U}p(\mathbf{U})\log |J|
\end{eqnarray}
\setlength{\arraycolsep}{5pt}where $|\cdot|$ denotes the absolute value and $J$ indicates the
determinant of Jacobian matrix, that is, $J=\mathrm{det}([\frac{\partial Z_{i}}{\partial U_{j}}]_{ij})$.
Note that the continuous and invertible $\sigma_{i}$ guarantees that $\log |\frac{\partial Z_{i}}{\partial U_{i}}|$
and $\log |J|$ are well-defined. To prove $I(\mathbf{Z})=I(\mathbf{U})$, we must prove
\begin{equation}
\int_{\mathbf{U}}\mathrm{d}\mathbf{U}p(\mathbf{U})\log |J|=\sum_{i}\int_{U_{i}} \mathrm{d}U_{i} p(U_{i}) \log |\frac{\partial Z_{i}}{\partial U_{i}}|
\end{equation}
It is straightforward since $\sigma_{i}$ is component-wise transfer function,
the Jacobian matrix is a diagonal matrix.
\end{proof}

\section{Nonparametric Entropy Estimation}
\label{appendix:estimator}
Estimating the entropy of high dimensional continuous random vectors such as $H(\mathbf{Z})$
is an challenging problem in both theories and applications. For a comprehensive review, please refer to \cite{beirlant_1997}. Here,
we use a nonparametric entropy estimator based on $k$-nearest neighbor distances, since
empirically it demonstrates better convergence properties \cite{kraskov_2004}
\begin{equation}
\hat{H}(\mathbf{Z}) = -\psi(k) + \psi(N) + \log c_{d} + \frac{d}{N}\sum_{i=1}^{N}\log \epsilon(i)
\label{equation:ee_knn}
\end{equation}
in which $d$ is the dimension of random vector, $N$ denotes the size of sample, $\psi(\cdot)$ is
a digamma function, $c_{d}$ equals $\pi^{\frac{d}{2}}/{\Gamma(1+\frac{d}{2})}/2^{d}$, $\epsilon(i)$
denotes twice the distance from the $i$-th example to its $k$-th neighbor. $k$-d tree can be used to
serve efficient nearest neighbor queries.

\section{Visualization}

\begin{figure*}
\centering
\includegraphics[width=0.9\textwidth]{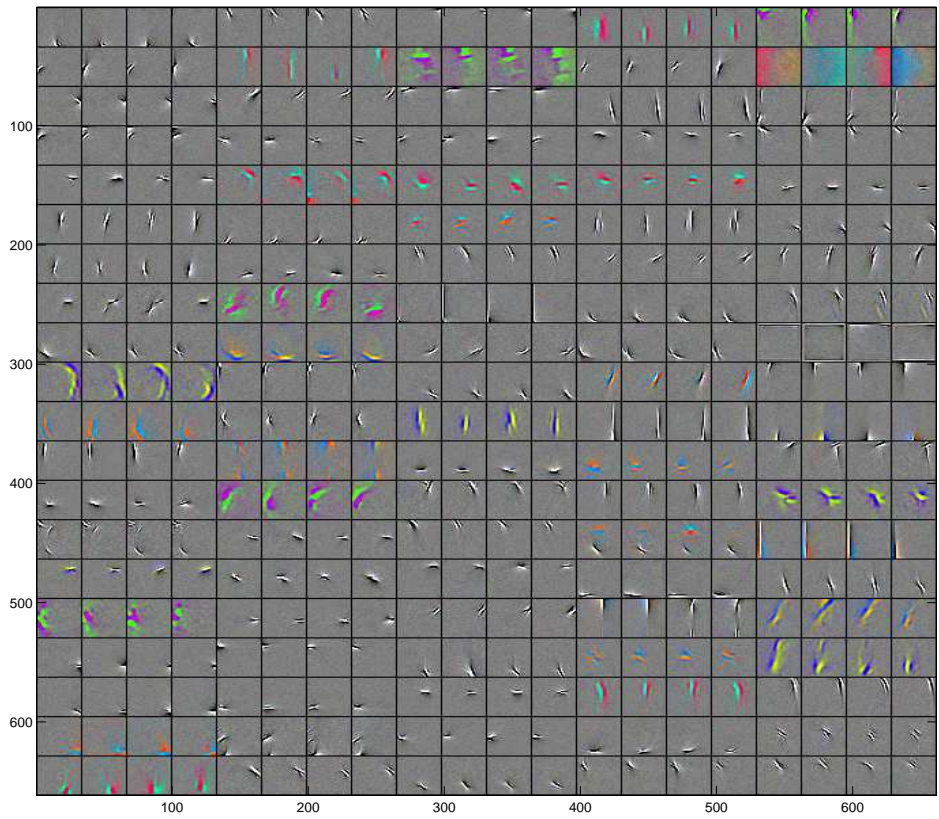}
\caption{A sample of the learned basis by group sparse coding. }
\label{layer1}
\end{figure*}

\end{document}